\documentclass[11pt]{article}
\usepackage{enumerate}
\usepackage{pdfsync}
\usepackage[OT1]{fontenc}
\usepackage{xcolor}
\usepackage{smile}
\usepackage[colorlinks,
            linkcolor=red,
            anchorcolor=blue,
            citecolor=blue
            ]{hyperref}
\usepackage{fullpage}
\usepackage[protrusion=true,expansion=true]{microtype}
\usepackage{pbox}
\usepackage{setspace}
\usepackage{tabularx}
\usepackage{float}
\usepackage{wrapfig,lipsum}
\usepackage{enumitem}
\usepackage{pgfplots}
\usetikzlibrary{arrows,shapes,snakes,automata,backgrounds,petri}

\newcommand{\algo}{\texttt{SI-CDB}}
\DeclareMathOperator{\Ber}{Ber}
\DeclareMathOperator{\Conf}{Conf}

\DeclareMathOperator{\Regret}{Regret}

\DeclareMathOperator{\KL}{KL}
\DeclareMathOperator{\elud}{\text{elud}}

\newcommand{\la}{\langle}
\newcommand{\ra}{\rangle}

\allowdisplaybreaks
\usepackage{colortbl}
\definecolor{LightCyan}{rgb}{0.8, 0.9, 1}
\definecolor{LightGray}{gray}{0.9}

\ifdefined\final
\usepackage[disable]{todonotes}
\else
\usepackage[textsize=tiny]{todonotes}
\fi
\makeatletter
\newcommand*{\rom}[1]{\expandafter\@slowromancap\romannumeral #1@}
\makeatother
\title{Saturation-Insensitive Dueling Bandits with\\General Function Approximation}
\author
{
    Chenggong Zhang\thanks{Department of Computer Science, University of California, Los Angeles, CA 90095, USA; e-mail: {\tt chenggong61@g.ucla.edu}}
    ~~~  
    Xuheng Li\thanks{Department of Computer Science, University of California, Los Angeles, CA 90095, USA; e-mail: {\tt xuheng.li@cs.ucla.edu}} 
    ~~~
    Qiwei Di\thanks{Department of Computer Science, University of California, Los Angeles, CA 90095, USA; e-mail: {\tt qiwei2000@cs.ucla.edu}}
	~~~
    Weitong Zhang \thanks{School of Data Science and Society, University of North Carolina, Chapel Hill, NC 27599, USA; e-mail: {\tt weitongz@unc.edu}}
    ~~~
	Quanquan Gu\thanks{Department of Computer Science, University of California, Los Angeles, CA 90095, USA; e-mail: {\tt qgu@cs.ucla.edu}}
}
\begin{document}
    \date{}
    \maketitle
\begin{abstract}
We study contextual dueling bandits with general function approximation under the \\ Bradley–Terry–Luce (BTL) preference model. A key challenge in this setting is the \textit{saturation} of the preference model: when the current reward model can already distinguish two actions with high confidence, the resulting preference feedback becomes weakly informative, making it difficult to further improve reward estimation. Consequently, existing sample-complexity analyses often depend on the inverse-derivative factor $1 / \sigma'[\Delta_{r^*}]$  which can be prohibitively large when the link function $\sigma$ saturates for large reward gaps $\Delta_{r^*}$. To address this issue, we introduce \algo, an algorithm that selects opponent arms using a carefully designed heuristic for arm selection. This design enables saturation-insensitive reward learning and recovers the near-optimal dependence for linear reward classes, eliminating the unfavorable $1/\sigma'(\cdot)$ factor. The core of our analysis is a localized Eluder dimension framework tailored to dueling bandits with general function approximation. Our theoretical results also explain why two-arm regret analysis is crucial for improving single-arm performance in dueling bandits.
\end{abstract}

\section{Introduction}
Dueling bandits~\citep{yue2012k, saha2021optimal, bengs2022stochastic} model online learning from binary preference feedback: in each round, the learner selects two actions and observes which one is preferred. This framework underlies the theory of reinforcement learning from human feedback (RLHF; \citealt{zhu2023principled, ji2023provable}) and of online preference-based fine-tuning of large language models (LLMs;~\citealt{ouyang2022training, xiong2024iterative}). The most widely used preference model is the Bradley--Terry--Luce (BTL) model~\citep{hunter2004mm, luce1959individual}, under which action $a$ is preferred to action $b$ in context $x$ with probability
\begin{align}\label{eq:intro-btl}
\PP(a \succ b \mid x) = \frac{\exp(r^*(x, a))}{\exp(r^*(x, a)) + \exp(r^*(x, b))} = \sigma\big(r^*(x, a) - r^*(x, b)\big),
\end{align}
where $r^*$ is an unknown reward function and $\sigma(z) = 1/(1+e^{-z})$ is the sigmoid \emph{link function}.

In these settings, simply collecting more pairwise comparisons is not enough, as the comparisons must also be informative for the learning objective.
Recent work on active preference learning for LLM alignment and RLHF supports this point, showing that adaptive query selection can improve label efficiency~\citep{muldrew2024active,das2025active,melo2024deep,ji2024reinforcement}. We study one basic mechanism behind uninformative comparisons under the BTL model: saturation of the sigmoid link.
To illustrate this issue, consider a context $x$ and two actions $(a,b)$ such that both the ground-truth reward $r^*$ and a current reward estimate $r_t$ clearly distinguish between them; i.e., $r^*(x, a) \gg r^*(x, b)$ and $r_t(x, a) \gg r_t(x, b)$.
In this regime, even if $r_t$ is still substantially different from $r^*$, the two induced preference models may be nearly indistinguishable on the comparison pair $(a,b)$, i.e., $\PP^*(a \succ b | x) \approx \PP_{r_t}(a \succ b | x) \approx 1$.
As a result, observing preferences between $a$ and $b$ provides little information for improving $r_t$. 
This saturation effect is reflected in the inverse-curvature dependence that appears in existing analyses~\citep{xie2025exploratory,xiong2024iterative}.
Under the BTL model, the derivative of the sigmoid link with respect to the reward gap $\Delta$ is $\sigma'(\Delta)=\sigma(\Delta)(1-\sigma(\Delta))$.
When the comparison is saturated, this derivative is exponentially small in the magnitude of the reward gap.
Consequently, analyses that rely on a uniform lower bound on the link curvature incur factors such as $\sup_{\Delta\in[-2B,2B]}1/\sigma'(\Delta)=\exp(\Theta(B))$, which can be prohibitively large in the reward scale $B$. We call a guarantee saturation-sensitive if its leading dependence scales with this worst-case inverse-curvature factor, and saturation-insensitive if the leading term avoids such dependence.

Such global inverse-curvature dependence appears in existing analyses
of contextual dueling bandits
\citep{saha2021optimal,bengs2022stochastic,di2024variance,li2024feel}.
This dependence is not unavoidable, as prior work has obtained sharper
regret bounds by exploiting local curvature in logistic bandits \citep{faury2020improved,abeille2021instance}, and a similar method has been developed for linear contextual dueling bandits \citep{di2025nearly}. However, these methods rely on a known linear parameterization to construct curvature-adaptive confidence sets and control exploration (see Appendix~\ref{app:rb-prior} for a detailed discussion). Their analyses therefore do not directly extend to general reward classes, leaving open how to obtain saturation-insensitive guarantees beyond the linear setting.

% We identify this saturation effect as the source of the inverse-curvature dependence that appears in existing analyses of contextual dueling bandits and RLHF with general function approximation. In particular, existing bounds often incur a factor of $\sup_{z\in[-B, B]} 1 / \sigma'(\Delta_{r^*}(z)) \sim \Theta(\exp(B))$ \citep{xie2025exploratory,xiong2023iterative}, which can be prohibitively exponential in the reward magnitude $B$.

To obtain widely applicable guarantees, we study it in the most general formulation of contextual dueling bandits considered here, allowing arbitrary realizable reward-function classes, and ask:
\begin{center}
\emph{Is it possible to get rid of the reward saturation phenomenon in\\  dueling bandits with general function approximation?}
\end{center}
In this paper, we answer this question affirmatively by presenting \algo. The core of the algorithm is to mitigate saturation through online exploration: after choosing a greedy target arm, it uses the current confidence set to select an optimistic comparator arm. This asymmetric arm-selection rule leads to a dueling-specific regret decomposition, which we combine with a localized eluder analysis over the induced excess-loss class.

We summarize the contributions of this paper as follows:
\begin{itemize}[leftmargin=*]
\item[1.] 
We propose \algo, a saturation-insensitive algorithm for contextual dueling bandits with general function approximation. The algorithm selects actions through two asymmetric function-action pairs. It first forms a target pair $(r_1,a)$ by choosing an empirical risk minimizer $r_1$ and playing its greedy action $a$. It then constructs, around $r_1$ in empirical loss, a confidence set of plausible reward functions. From this confidence set, the algorithm forms a comparator pair $(r_2,b)$ by jointly choosing a reward function and a comparator action. Given the target pair $(r_1,a)$ and the context $x$, this comparator pair is selected by maximizing
$J(b,r)=2r(x,b)-r_1(x,b)-2r(x,a)$.
The coefficients in this objective are chosen so that, at each round $t$, the two-arm regret decomposes into played-pair reward-gap errors of the form $|\Delta_{r_j}(x,a,b)-\Delta_{r^*}(x,a,b)|$ for $j\in\{1,2\}$. This decomposition is the key bridge to the localized excess-loss analysis.
% We propose \algo, a saturation-insensitive algorithm for contextual dueling bandits that works with the general reward class. Crucially, after greedily choosing the main arm and reward function $(a_t, r_1^t)$ at step $t$, the comparator arm and reward function $(b_t, r_2^t)$ are jointly selected by maximizing the objective $J_t(b, r)=2r(x_t, b)-r_1^t(x_t, b)-2r(x_t, a_t)$. The specific design of the objective function (notably the coefficients), which is a novel contribution of this paper, enables the decomposition of regret into the reward error terms $|\Delta_{r_j^t}(z_t)-\Delta_{r^*}(z_t)|$ for $j\in\{1, 2\}$, and further facilitates the localized analysis.
\item[2.] We derive a regret bound for \algo \ based on the localized analysis. We partition all steps into two parts based on whether the reward error terms $|\Delta_{r_j^t}(z_t)-\Delta_{r^*}(z_t)|$ are small. Our localized analysis shows that (i) the total regret over steps with small reward errors, which is the main contributor to the cumulative regret bound, depends only polynomially on $B$, and (ii) the number of steps with large reward errors, despite the exponential dependence on $B$, depends only logarithmically on the number of steps $T$ and is therefore a lower-order term when $T\to\infty$. Compared with works on both the linear reward function~\citep{saha2021optimal, bengs2022stochastic} and the general reward function class~\citep{li2024feel}, the regret of \algo \ does not contain terms that scale with both $\mathrm{poly}(T)$ and $e^{\Theta(B)}$.

\end{itemize}

\noindent\textbf{Notations.}
We use normal-face letters to denote scalars, lower-case bold-face letters to denote vectors, and upper-case bold-face letters to denote matrices.
Let $\ind[\cdot]$ denote the indicator function.
For a positive integer $N$, let $[N]$ denote the set $\{1, 2, \dots, N\}$.
We use standard asymptotic notation $O(\cdot)$, $\Omega(\cdot)$, and $\Theta(\cdot)$, with $\tilde O(\cdot)$, $\tilde\Omega(\cdot)$, and $\tilde\Theta(\cdot)$ hiding logarithmic factors. For $p \in [0,1]$, we use $\Ber(p)$ to denote the Bernoulli distribution with mean $p$. For a vector $\xb$, $\|\xb\|_2$ denotes its $\ell_2$-norm.

\section{Related Work}
\noindent\textbf{Dueling bandit and contextual preference learning.} Dueling bandits extend the standard bandit framework by replacing direct scalar reward observations with pairwise preference feedback between selected actions, which is a setting initiated by~\citet{yue2012k}. As surveyed by~\citet{bengs2021preference}, prior work on dueling bandits has studied a range of optimality notions, including Condorcet winners~\citep{zoghi2014relative,komiyama2015regret}, Copeland winners~\citep{zoghi2015copeland,wu2016double,komiyama2016copeland}, Borda winners~\citep{jamieson2015sparse,falahatgar2017maximum,heckel2018approximate,saha2021adversarial}, and von Neumann winners~\citep{ramamohan2016dueling,dudik2015contextual,balsubramani2016instance}. Contextual dueling bandits incorporate side information and typically require structure on the preference model. \citet{saha2021optimal} studied stochastic contextual preference bandits with a logistic link function. More recent works provide variance-aware guarantees~\citep{di2024variance}, corruption-robust guarantees~\citep{cheng2025learning,di2025nearly}, oracle-efficient algorithms under realizability~\citep{saha2022efficient}, active-query guarantees for general function class~\citep{sekhari2023contextual}, Thompson sampling-style algorithm for dueling bandits \citep{li2024feel} and Borda-regret guarantees for generalized linear dueling bandits~\citep{wu2024borda}. 

\noindent\textbf{Logistic and generalized linear bandits.} Generalized linear bandits, including logistic bandits as a canonical example, model the expected reward through a nonlinear link function applied to a linear predictor.
 Early UCB-style algorithms for generalized linear bandits incur curvature factors that can be large when the link derivative is small ~\citep{filippi2010parametric,li2017provably}. Subsequent work on logistic bandits showed that such worst-case curvature
dependence is not always unavoidable.
\citet{faury2020improved} used self-concordance to sharpen this dependence.
\citet{abeille2021instance} developed instance-dependent minimax guarantees
based on local curvature.
\citet{faury2022jointly} provided a more computationally efficient algorithm. \citet{liu2026learning} further studied the simple-regret objective.
They showed that the informativeness of an action depends on the local
curvature of the sigmoid.
Actions with the highest reward therefore need not be the most informative. The same saturation phenomenon appears in dueling bandits: a global lower bound on the sigmoid slope over the reward range may be exponentially small. \citet{chen2026avoiding} avoid this exponential dependence in online RLHF by adopting a preference-based notion of regret. Their objective compares the optimal and learned responses against a fixed sampler policy, rather than measuring standard reward-based regret. Therefore, removing the exponential dependence from the leading regret term for general reward function classes remains an open problem. 

\noindent\textbf{General function approximation.} The eluder dimension was introduced by~\citet{russo2013eluder} to quantify exploration complexity for general function approximation. Subsequent work on nonlinear bandits and reinforcement learning measures exploration complexity through sums of uncertainty-weighted estimation errors, leading to generalized variants of the eluder dimension~\citep{agarwal2023vo,ye2023corruption,di2024pessimistic,zhao2024nearly}. However, global eluder dimensions can be pessimistic for generalized linear models because they reflect worst-case curvature. The recent localized-eluder framework of~\citet{bakhtiari2026eluder} addresses this by measuring complexity on localized expected excess-loss scales and treating non-local rounds separately. 
% Related first-order and small-loss analyses\citep{wang2023benefits} use Decision-Estimation or surrogate-loss conditions ~\citep{foster2018contextual} to convert estimation error into regret~\citep{wang2024more,ayoub2024switching}.
Our work imports the localization perspective into contextual dueling bandits. The main difference is structural: rather than a single-action cost regret bridge, our algorithm is designed to yield a dueling-specific decomposition in terms of two played reward-gap errors, which are then controlled by log-loss excess and localized eluder arguments.

%=====================================
\section{Preliminaries}\label{sec:preliminaries}

In this work, we study the contextual dueling bandit with a general reward function class.
Let $\cX$ be the context space and $\cA$ be the finite action space.
At each round $t\in[T]$, the learner observes the context $x_t\in\cX$,
selects two actions $a_t, b_t\in\cA$, and observes a binary preference feedback $o_t\in\{0,1\}$, where $o_t=1$ means that $a_t$ is preferred
to $b_t$, denoted as $a_t\succ b_t$.
We use $z=(x, a, b)$ to denote the context-action triplet, and write
$z_t=(x_t,a_t,b_t)$ for the query played at round $t$. Suppose that $r^*:\cX\times\cA\to[-B, B]$ is the ground truth reward function. The preference probability, i.e., $\PP(o_t=1)$, is determined by the difference of the ground truth reward $\Delta_{r^*}(z_t)=r^*(x_t, a_t)-r^*(x_t, b_t)$ and the link function $\sigma:\RR\to[0, 1]$, with
\begin{align*}
\PP(o_t=1|z_t)=\PP(a_t\succ b_t|z_t)=\sigma(\Delta_{r^*}(z_t)).
\end{align*}
We apply the Bradley-Terry-Luce (BTL) model~\citep{hunter2004mm, luce1959individual} as the link function, where the preference probability satisfies
\begin{align*}
\PP(o_t=1|z_t)=\frac{\exp(r^*(x_t, a_t))}{\exp(r^*(x_t, a_t))+\exp(r^*(x_t, b_t))}=\frac{1}{1+\exp(-\Delta_{r^*}(z_t))}.
\end{align*}
Therefore, in the BTL model, the link function is $\sigma(v)=1/(1+e^{-v})$. We make the following realizability assumption on the ground-truth reward function $r^*$.
\begin{assumption}\label{assumption1}
There exists a reward function class $\cR\subseteq\{r:\cX\times\cA\to[-B,B]\}$, such that the ground-truth reward function $r^* \in \cR$. 
\end{assumption}
\begin{example}
\label{example:linear}
We consider the linear reward function class as an important case: The reward function $(x, a)\mapsto r_{\btheta}(x, a)=\langle\btheta, \bmu(x, a)\rangle$ is parameterized by $\btheta\in\Theta\subset\{\btheta\in\RR^d: \|\btheta\|_2\le B\}$, and $\bmu(x,a)\in \mathbb B_2^d(0,1)$ is the feature mapping. Suppose the ground truth reward function is $r^*=r_{\btheta^*}$ with $\btheta^*\in\Theta$. Therefore, the reward function satisfies $|r_{\btheta}(x, a)|\le\|\btheta\|_2\cdot\|\bmu(x, a)\|_2\le B$ by Cauchy-Schwarz inequality.
\end{example}
Beyond this linear example, our framework also accommodates nonlinear
reward functions. This flexibility comes at the cost of a more
involved analysis, but avoids the potentially restrictive assumption
that the latent reward is linear in a known feature map.

In this work, we assume $B\ge 1$, but do not impose a constant upper bound on $B$. 
Consequently, the latent reward gap can be large: there may exist a triplet 
$z=(x,a,b) \in \cX \times \cA \times \cA$ such that
\begin{align*}
|\Delta_{r^*}(z)| \simeq 2B .
\end{align*}
Such large gaps may drive the sigmoid link function into its saturation regime, where the gradient $\sigma'$ becomes exponentially small.
Indeed, when $u=\Delta_{r^*}(z)$ has large magnitude, the preference probability $\sigma(u)$
is close to either 0 or 1. 
Thus the comparison is nearly deterministic, while the link itself becomes nearly flat, with its derivative $\sigma'(u)=\sigma(u)(1-\sigma(u))\asymp e^{-|u|} \to 0$.
This places us outside the regime commonly considered in analyses that assume the link function has a uniformly bounded inverse derivative. 
In particular, prior works often introduce a condition of the form
\begin{align*}
\kappa:=\sup_{z\in\mathcal X\times\mathcal A\times\mathcal A}\Big\{\frac{1}{\sigma'[\Delta_{r^*}(z)]} \Big\}<\infty.
\end{align*}
When $\sigma$ is sigmoid, however, $\sigma'(u)\asymp e^{-|u|}\to 0$ when $|u| \to \infty$. Previously, prior analyses often treat $\kappa$ as a benign problem-dependent constant. Thus, if $|\Delta_{r^*}(z)|\simeq 2B$, then
$\kappa \simeq e^{2B}$ can be exponentially large in the reward scale.

Our goal is to minimize the cumulative two-armed regret:
\begin{align*}
\Regret(T)=\sum_{t=1}^TR_t,\quad\text{where}\quad R_t=2r^*(x_t, a_t^*)-r^*(x_t, a_t)-r^*(x_t, b_t).
\end{align*}
With the time-varying context, the optimal action in step $t$, i.e., $a_t^*=\argmax_{a\in\cA}r^*(x_t, a)$, can also be time-varying. The term $R_t$ calculates the gaps of the reward of the optimal action $a_t^*$ and the selected actions $a_t, b_t$.
%=====================================

\section{Methodology}\label{sec:m}

We now present \algo{} (Algorithm~\ref{algo:sicdb}).
The key component of \algo{} is the asymmetric selection of the two actions: the target action $a_t$ exploits the current reward estimate (Lines~\ref{alg:line5}--\ref{alg:line6}), while the comparator action $b_t$ is chosen optimistically to explore (Line~\ref{alg:line8}).

To select the target action, the algorithm first estimates the reward function from the preference feedback collected so far.
For a candidate preference probability $p\in(0,1)$, the Bernoulli log loss is
\begin{align}\label{eq:def_logistic}
\ell(o, p)=-o\log p-(1-o)\log(1-p).
\end{align}
Given the past queries $z_1,\dots,z_{t-1}$ and feedback $o_1,\dots,o_{t-1}$, the cumulative log loss of a reward function $r$ is
\begin{align}\label{eq:L}
L_t(r)=\sum_{i=1}^{t-1}\ell\big(o_i, \sigma[\Delta_r(z_i)]\big).
\end{align}
The algorithm computes the maximum likelihood estimate $r_1^t \in \argmin_{r\in\cR} L_t(r)$ (Line~\ref{alg:line5}).
Given $r_1^t$, the algorithm selects the target action greedily, i.e., $a_t \in \argmax_{a \in \cA} r_1^t(x_t, a)$.

To select the comparator action, the algorithm first constructs a confidence set around $r_1^t$ using the same loss $L_t$:
\begin{align*}
\Conf_t = \{r \in \cR : L_t(r) \le L_t(r_1^t) + \beta_t\},
\end{align*}
which contains the ground-truth reward $r^*$ with high probability.
Similar constructions are widely used in model-based RL \citep{zhan2022pac, liu2023optimistic, wang2025model} for optimistic exploration with general function approximation.
The algorithm then jointly selects the comparator action $b_t$ and an optimistic reward function $r_2^t\in\Conf_t$ by maximizing the following objective over $\cA\times\Conf_t$.
\begin{align}
\label{eq:J-simplified}
    J_t(b, r) := 2\cdot r(x_t, b)~-1\cdot r_1^t(x_t, b)~-2\cdot r(x_t, a_t),
    \qquad
    (b_t, r_2^t) \in \argmax_{b \in \cA,\, r \in \Conf_t} J_t(b, r),
\end{align}
Maximizing over $r\in\Conf_t$ plays the role of an uncertainty bonus without requiring an explicit feature representation. To explain the coefficients $(2,-1,-2)$, we rewrite $J_t(b,r)$ as
\begin{align*}
J_t(b, r)=\underbrace{r_1^t(x_t, b)+2[\Delta_{r_1^t}(x_t, a_t, b)-\Delta_r(x_t, a_t, b)]}_{\tilde J_t(b, r)}-2r_1^t(x_t, a_t).
\end{align*}
The term $-2r_1^t(x_t,a_t)$ does not depend on $b$ or $r$.
Thus, maximizing $J_t(b,r)$ is equivalent to maximizing $\bar J_t(b,r)$.
Its two terms correspond to exploitation and exploration, respectively:
\begin{itemize}[leftmargin=*]
\item 
The first term favors a comparator with high estimated reward.
It therefore encourages exploitation.
\item 
The second term favors disagreement in the played gap among plausible
models.
\end{itemize}
The coefficients are selected to produce the two-gap decomposition in Lemma~\ref{lemma:per_step_regret}, not merely as a heuristic. We summarize the procedue in Algorithm~\ref{algo:sicdb}. In Appendix~\ref{app:rb-linear}, we will provide a comparison of \algo \ with prior work \citep{bengs2022stochastic, di2025nearly} when restricted to the linear setting.
\begin{algorithm}[H]
\caption{Saturation-Insensitive Contextual Dueling Bandits (\algo)}
\label{algo:sicdb}
\begin{algorithmic}[1]
\STATE \textbf{Input:} Reward function class $\cR$, action space $\cA$, confidence widths $(\beta_t)_{t\le T}$.
\STATE Initialize $L_1(r) = 0$ for all $r \in \cR$.
\FOR{$t = 1, \ldots, T$}
    \STATE Observe context $x_t$.
    \STATE Select $r_1^t \in \argmin_{r \in \cR} L_t(r)$. \label{alg:line5}
    \STATE Select target action $a_t \in \argmax_{a \in \cA} r_1^t(x_t, a)$.\label{alg:line6}
    \STATE Construct confidence set $\Conf_t = \{r \in \cR : L_t(r) \le L_t(r_1^t) + \beta_t\}$.
    \STATE Select comparator pair of action and reward function: $(b_t, r_2^t) \in \argmax_{b \in \cA,r \in \Conf_t} J_t(b, r)$ where $J_t(b, r)$ is defined in \eqref{eq:J-simplified}.\label{alg:line8}
    \STATE Play $(a_t, b_t)$, and observe preference $o_t\in\{0, 1\}$.
    \STATE Update $L_{t+1}(r) = L_t(r) + \ell(o_t, \sigma(\Delta_r(z_t)))$ for all $r \in \cR$, where $\ell(\cdot, \cdot)$ is defined in~\eqref{eq:def_logistic}
\ENDFOR
\end{algorithmic}
\end{algorithm}
%</R1-A9-NEW-LOCATION>

\section{Main Results}\label{sec:main_results}

In this section, we present the regret guarantee of \algo{}. We first introduce the complexity measures used in our analysis, namely the covering number and the localized eluder dimension of the excess-loss class induced by $\cR$ (Section~\ref{sec:complexity}). We then state the regret bound for general reward classes and specialize it to linear rewards (Section~\ref{sec:regret}).

\subsection{Covering Number and Localized Eluder Dimension}\label{sec:complexity}
Our regret bound depends on two complexity measures of the
excess-loss classes induced by the reward class: a covering number
and a localized eluder dimension. We first introduce these classes,
which reflect the preference feedback available to the learner.

In standard contextual bandits with reward feedback, candidate
functions are distinguished through differences in predicted rewards.
In dueling bandits, however, the learner observes preferences, so
statistical distinguishability depends on the Bernoulli distributions
induced by the reward functions through the comparison link.
We capture this distinction using the excess log loss.
For each $r\in\cR$ and query $z=(x,a,b)$, define the excess log loss
and its expectation under the true preference distribution as
\begin{align}
\notag
\phi_r(o, z)&=\ell(o, \sigma[\Delta_r(z)])-\ell(o, \sigma[\Delta_{r^*}(z)]),\\\label{eq:excess}
\bar\phi_r(z)&=\EE_{o\sim\Ber(\sigma[\Delta_{r^*}(z)])}[\phi_r(o, z)]=\KL\big(\Ber(\sigma[\Delta_{r^*}(z)])\|\Ber(\sigma[\Delta_r(z)])\big).
\end{align}
Thus, $\bar\phi_r(z)$ equals the Kullback--Leibler divergence from
the true preference distribution to that induced by $r$ at query $z$.
We define the corresponding classes of realized and expected excess loss as
\begin{align}\label{eq:def_func_class_w_loss}
\Phi(\cR)=\{\phi_r: r\in\cR\},\quad\bar\Phi(\cR)=\{\bar\phi_r: r\in\cR\}.
\end{align}
The two classes enter our analysis in different ways. The covering number of $\Phi(\cR)$ to obtain concentration bounds that hold uniformly over all candidate reward functions, while the localized eluder dimension of $\bar\Phi(\cR)$ to bound the cumulative expected excess loss of the selected models on the queries made by the algorithm.

\noindent\textbf{Covering number.} Let $\cN(\cF,\epsilon,\|\cdot\|)$ denote the covering number of
$\cF$ at radius $\epsilon$ under the norm $\|\cdot\|$.
For the realized excess-loss class, we use
\begin{align}\label{eq:def_covering}
N_T=\cN(\Phi(\cR), 1/T, \|\cdot\|_\infty).
\end{align}
Here $\|\cdot\|_\infty$ is the supremum norm over observations
$(o,z)$.

\noindent\textbf{Localized eluder dimension.} To avoid exponential dependence on $B$ in the leading regret term, we seek to control regret using local rather than worst-case curvature. The localized eluder dimension of~\citet{bakhtiari2026eluder} provides a tool for controlling expected excess loss at local scales. Here, we provide the formal definition.
\begin{definition}[Localized Eluder Dimension, \citealt{bakhtiari2026eluder}]
Let $\cZ$ be a set, $\Psi$ be a class of real-valued functions on $\cZ$. Let $z = (z_1,z_2,\ldots,z_n)$ be a length-$n$ sequence in $\cZ$. The localized Eluder dimension is defined as follows:
\begin{itemize}[leftmargin=*]
    \item We say that $x \in \cZ$ is $\epsilon$-independent of $z$ with respect to $\Psi$ if there exists $\psi \in \Psi$ such that 
    \begin{align*}
        \sum_{t=1}^n|\psi(z_t)| \le \epsilon,\quad\text{and}\quad|\psi(x)| > \epsilon.
    \end{align*}
    \item We say that $z$ is an $\epsilon$-eluder sequence with respect to $\Psi$ if, for every $t \le n$, the point $z_t$ is $\epsilon$-independent of $z_1,\ldots, z_{t-1}$ with respect to $\Psi$.
    \item The $(\epsilon,\eta)$-localized eluder dimension $\text{dim}_{\text{elud}}^\eta(\epsilon,\Psi)$ of $\Psi$ is the maximum length of an $\omega$-eluder sequence with respect to $\Psi$ over all $\omega \in [\epsilon,\eta]$.
\end{itemize}
\end{definition}
In this work, we consider $d^{\epsilon,\eta}_{\elud}(\cR) := \text{dim}_{\text{elud}}^\eta\big(\epsilon,\bar \Phi(\cR)\big)$. We write it as $d^{\eta}_{\elud}(\cR)$ when $\epsilon=1/T$.

The upper scale $\eta$ restricts the independence thresholds considered in the definition, allowing us to measure complexity at the excess-loss scales relevant to our regret analysis. However, the analysis of~\citet{bakhtiari2026eluder} does not directly address two-arm reward regret under pairwise preference feedback. Our asymmetric selection rule and the resulting regret decomposition enable us to extend the localized analysis to this setting, as detailed in Section~\ref{sec:proof_overview}.
\begin{remark}
In prior work, complexity measures can be distinguished according to whether they are determined by the hypothesis class alone or also depend on the underlying problem instance. The first category includes the standard eluder dimension and its weighted generalization \citep{russo2013eluder,agarwal2023vo}, which is defined by differences between functions in the hypothesis class and do not depend on the ground-truth function. The second category includes the Bellman eluder dimension \citep{jin2021bellman}, defined through residuals involving the \emph{real} Bellman operator, and the complexity measure of~\citet{wang2025model} for the model-based setting. Our measure belongs to the second category, as $\phi_r$ is defined relative to the true reward function $r^*$.
\end{remark}

\subsection{Regret Bound of \algo}\label{sec:regret}
With the covering number and the localized eluder dimension, we now provide the regret upper bound of \algo:
\begin{theorem}
\label{thm:main}
For $\delta\in(0,1)$, choose the confidence radius as $\beta_t = O\!\big(B\log({N_T}/{\delta})\big)$. 
where $N_T$ is defined in \eqref{eq:def_covering}. Then, with probability at least $1-\delta$, the regret of Algorithm \ref{algo:sicdb} can be upper bounded by 
    \begin{align*}
    \Regret(T) \le \inf_{u>0} \bigg[3 \sqrt{\frac{2T \Gamma_u}{\sigma'(4u)}}+4B\bigg(\frac{8\beta_T}{u^2 \sigma'(2B)} + 1\bigg) d^{\theta_u,\theta_u}_{\elud}(\cR)\bigg]+ 8B.
\end{align*}
where 
\begin{align*}
    \Gamma_u = O\!\left(\beta_T\log(1+BT)\,d^{\tau_u}_{\elud}(\cR)+B\right),
    \qquad \tau_u = \min\{2B,u^2/8\},
    \qquad \theta_u = \frac{\sigma'(2B)}{2}u^2.
\end{align*}
\end{theorem}
In Theorem~\ref{thm:main}, $u$ is an analysis parameter that
sets the localization scale and balances the following two
contributions to regret:
\begin{itemize}[leftmargin=*]
\item The first term controls regret on localized rounds and
depends on the local curvature $\sigma'(4u)$.
For $u=O(1)$, this curvature is bounded away from zero,
so the term avoids the worst-case inverse-curvature factor.
Its remaining dependence on the reward class is captured by
$\Gamma_u$.
\item The second term accounts for non-localized rounds.
It retains the potentially exponential factor $1/\sigma'(2B)$,
but has no explicit $\sqrt{T}$ multiplier.
Its dependence on $T$ enters through $\beta_T$ and
$d_{\elud}^{\theta_u,\theta_u}(\cR)$.
\end{itemize}
The infimum over $u>0$ allows us to select the best tradeoff between these two contributions. Thus, the bound separates the local-curvature contribution from an additive term that retains the worst-case curvature dependence. 

To demonstrate the effectiveness of our results, we consider the linear function class. We firstly establish the bound for the covering number and the localized eluder dimension of the linear reward function class:
\begin{proposition}
    The uniform-covering number can be bounded by
\begin{align*}
    \log N_T \le d \log(1+8BT).
\end{align*}
Moreover, there exists a universal constant $C>0$
such that for every $0<\epsilon\le\eta\le{\sigma'(2B)}/{4}$, 
we have
\begin{align*}
    d_{\elud}^{\epsilon,\eta}(\cR)
    \le Cd\log\Big(1+\frac{B^2}{\epsilon}\Big) .
\end{align*}
\end{proposition}
For the derivation of these results, please refer to Appendix \ref{sec:app-glm}. 

Setting $u=\sqrt{\sigma'(2B)}$ in Theorem~\ref{thm:main} ensures that both $\tau_u$ and $\theta_u$ are at most $\sigma'(2B)/4$, as required by the preceding proposition. Then, $\theta_u = {[\sigma'(2B)]^2}/{2}$. As a result, we have
\begin{corollary}
\label{cor:linear}
With probability at least $1-\delta$, the regret of Algorithm 1 can be upper bounded by
\begin{align*}
    \Regret(T) \le \tilde O\bigg(d\sqrt{BT} + \frac{d^2B^2}{[\sigma'(2B)]^2}\bigg).
\end{align*}
\end{corollary}
The dependence of $d$ and $T$ in the leading term $\tilde O(d\sqrt{BT})$ matches the lower bound proved in \citet{li2024feel}. Moreover, it improves the dependence of $B$ to $\sqrt{B}$ compared with \citet{di2025nearly}, while our results can further extend to general function approximation.
% We use the Bernoulli log-loss
% \begin{align}
%     \ell(o,p)
%     &:=
%     -o\log p-(1-o)\log(1-p) .
%     \label{eq:logloss-def}
% \end{align}
% Writing $p=\sigma(v)$ and introducing the logistic potential
% \begin{align}
%     \psi(v)
%     &:=
%     \log(1+e^v),
%     \label{eq:logistic-potential-def}
% \end{align}
% we have $\ell(o,\sigma(v))=\psi(v)-ov$, with
% $\psi'(v)=\sigma(v)$ and $\psi''(v)=\sigma'(v)$.\\

%=====================================
\section{Proof Overview}
\label{sec:proof_overview}
Before presenting the proof, we emphasize the role of two-armed regret in our analysis. The key insight is that, even when only the performance of the first arm matters, controlling two-armed regret provides the query-gap localization needed to avoid worst-case inverse-curvature dependence. We illustrate this connection using the linear reward class in Example~\ref{example:linear}.

\noindent\textbf{Intuition.}
Let $\db_t=\bmu(x_t,a_t)-\bmu(x_t,b_t)$ denote the feature difference
of the queried pair. The Hessian of its logistic loss, evaluated at
$\btheta^*$, is
\[
    \sigma'(\db_t^\top\btheta^*)\,\db_t\db_t^\top.
\]
The local curvature factor $\sigma'(\db_t^\top\btheta^*)$ can be
much larger than its uniform lower bound $\sigma'(2B)$.
In particular, on rounds with small two-armed regret, both queried
actions have near-optimal rewards, and hence
\[
    \big|\db_t^\top\btheta^*\big|
    =|r^*(x_t,a_t)-r^*(x_t,b_t)|
    \le R_t.
\]
As this gap approaches zero, $\sigma'(\db_t^\top\btheta^*)$
approaches $1/4$. Thus, on rounds with small two-armed regret, the local curvature
is bounded away from zero even when the worst-case lower bound
$\sigma'(2B)$ is exponentially small.
In contrast, small one-arm regret does not ensure benign curvature:
even if $a_t$ is optimal, a suboptimal comparator $b_t$ can leave
the queried reward gap large and the comparison highly saturated.

To use this intuition in the proof, we must identify rounds on which
the two-armed regret is small. Our asymmetric selection rule makes
this possible by bounding $R_t$ in terms of the gap-estimation errors
at the queried pair. When these estimation errors are small, the regret bound ensures
that both the true and estimated reward gaps lie in a region of
benign curvature. We formalize this argument through an asymmetric
regret decomposition and a localized eluder analysis.

\noindent\textbf{Key Technique I: Asymmetric Regret Decomposition.}
Algorithm~\ref{algo:sicdb} maintains two reward functions with distinct
roles. The MLE $r_1^t$ determines the greedy target action $a_t$,
while $r_2^t\in\Conf_t$ is selected jointly with the comparator $b_t$
by maximizing the asymmetric objective $J_t$.
We show how these two selection rules combine to bound two-armed
regret by estimation errors at the queried pair.

Assume $r^*\in\Conf_t$, as justified by the concentration argument
below. Then $(a_t^*,r^*)$ is feasible for the comparator optimization,
so $J_t(a_t^*,r^*)\le J_t(b_t,r_2^t)$. Expanding this inequality gives
\begin{align*}
    &2r^*(x_t,a_t^*)-r_1^t(x_t,a_t^*)-2r^*(x_t,a_t)\\
    &\qquad\le 2r_2^t(x_t,b_t)-r_1^t(x_t,b_t)-2r_2^t(x_t,a_t).
\end{align*}
Substituting into the definition of $R_t$ and using the greedy
optimality condition $r_1^t(x_t,a_t^*)\le r_1^t(x_t,a_t)$, we obtain
\begin{align*}
    R_t
    &\le r_1^t(x_t,a_t^*)-r_1^t(x_t,b_t)
        +\Delta_{r^*}(z_t)-2\Delta_{r_2^t}(z_t)\\
    &\le \Delta_{r_1^t}(z_t)+\Delta_{r^*}(z_t)-2\Delta_{r_2^t}(z_t).
\end{align*}
Thus the greedy selection removes the dependence on $a_t^*$,
and the coefficients in $J_t$ allow the remaining terms to be
written as two gap-estimation errors. Taking absolute values yields
the following decomposition.
\begin{lemma}[Informal]
\label{lemma:main_per_step_regret}
For any round $t$ with $r^*\in\Conf_t$, the actions and reward
functions selected by Algorithm~\ref{algo:sicdb} satisfy
\begin{align*}
    R_t\le
    \underbrace{\big|\Delta_{r_1^t}(z_t)-\Delta_{r^*}(z_t)\big|}_{e_{1,t}}
    +2\underbrace{\big|\Delta_{r^*}(z_t)-\Delta_{r_2^t}(z_t)\big|}_{e_{2,t}}.
\end{align*}
\end{lemma}
Crucially, both errors concern the comparison actually queried:
no estimation error at the unknown optimal action remains.
The decomposition therefore connects reward regret to the expected
excess losses $\bar\phi_{r_j^t}(z_t)$ defined in~\eqref{eq:excess},
which measure how well the selected models predict the observed
preferences. It also provides the bound $R_t\le3\max_j e_{j,t}$
needed to turn small estimation errors into query-gap localization.
\textbf{Thus, small gap-estimation errors ensure that both queried actions are near-optimal.}
We use both consequences in the next part.

\noindent\textbf{Key Technique II: Localized Eluder Analysis.}
We combine the regret decomposition with the localized eluder tools
to control regret using the curvature of the queried comparisons.
For a threshold $u>0$, define
\begin{align*}
    E_t:=\max_{j\in\{1,2\}}e_{j,t},\qquad
    I_T:=\{t\in[T]:E_t\le u\},\qquad
    I_T^c:=[T]\setminus I_T.
\end{align*}
This partition is used only in the analysis: the algorithm neither
observes $E_t$ nor needs to identify the localized rounds. The same
sequence of actions is analyzed for every $u>0$, allowing us to
optimize the resulting bound over $u$.

On the localized rounds $I_T$, the regret decomposition gives $R_t\le3E_t\le3u$. Because $R_t$ sums the suboptimalities of both queried actions, it also bounds their true reward gap. Consequently, for $j\in\{1,2\}$,
\begin{align*}
    |\Delta_{r^*}(z_t)|&\le R_t\le3u,\\
    |\Delta_{r_j^t}(z_t)|&\le |\Delta_{r^*}(z_t)|+e_{j,t}\le4u.
\end{align*}
Thus both the true and estimated gaps lie in $[-4u,4u]$.
\textbf{Controlling both arms is what bounds the true query gap here.} Since the second derivative of the logistic loss is $\sigma'$,
its curvature between the true and estimated gaps lies between
$\sigma'(4u)$ and $1/4$. A second-order expansion therefore gives
\begin{align*}
    \frac{\sigma'(4u)}{2}e_{j,t}^2
    \le \bar\phi_{r_j^t}(z_t)
    \le \frac18 e_{j,t}^2
    \le \frac{u^2}{8}.
\end{align*}
The lower bound converts expected excess loss into regret using local
curvature, while the upper bound restricts the excess-loss scale
needed for the localized eluder analysis. It remains to control the
accumulation of these losses along the adaptively selected queries.

For this step, we first use the uniform concentration bound
of~\citet{bakhtiari2026eluder}. With probability at least $1-\delta$, uniformly over $r\in\cR$ and $t\in[T]$,
\begin{align*}
    \sum_{i<t}\bar\phi_r(z_i)
    \le2\sum_{i<t}\phi_r(o_i,z_i)+2\beta_t.
\end{align*}
Since expected excess losses are nonnegative, this event ensures $r^*\in\Conf_t$. Moreover, the definition of the confidence set implies that each selected estimator has bounded historical expected excess loss:
\begin{align*}
    \sum_{i<t}\bar\phi_{r_j^t}(z_i)\le4\beta_T,
    \qquad j\in\{1,2\}.
\end{align*}
Although the selected model changes across rounds, each model must
remain consistent with the previous comparisons. The localized
eluder dimension controls how often such a model can fit the history
yet incur substantial expected excess loss on a new query. This
connects the historical loss bound to the accumulation of losses
along the algorithm's queries.

This bound also holds when the history is restricted to $I_T$. Since the excess losses on $I_T$ are localized at scale $\tau_u$,
the historical loss bound allows us to control their accumulation
using the localized eluder dimension, giving
$\sum_{t\in I_T}\bar\phi_{r_j^t}(z_t)=O(\Gamma_u)$.
The local curvature bound and Cauchy--Schwarz then give
\begin{align*}
    \sum_{t\in I_T}R_t
    \le\sum_{t\in I_T}(e_{1,t}+2e_{2,t})
    =O\!\left(\sqrt{\frac{T\Gamma_u}{\sigma'(4u)}}\right).
\end{align*}
For the remaining rounds $I_T^c$, we bound their number instead.
On each such round, at least one error exceeds $u$. The global curvature bound implies that its expected excess loss exceeds $\theta_u=\sigma'(2B)u^2/2$. The historical loss bound and the localized eluder dimension then control the number of such rounds:
\begin{align*}
    |I_T^c|=O\!\left(
    \left(1+\frac{\beta_T}{\sigma'(2B)u^2}\right)
    d_{\elud}^{\theta_u,\theta_u}(\cR)+1\right).
\end{align*}
Using $R_t\le4B$ on these remaining rounds and optimizing over $u$ yields the two contributions in Theorem~\ref{thm:main}. The worst-case inverse-curvature factor enters the bound for the remaining rounds, while the localized $\sqrt T$ term uses $\sigma'(4u)$ and the localized complexity encoded in $\Gamma_u$.

\section{Conclusion} \label{sec:con}
We studied contextual dueling bandits with general reward function approximation under BTL preference feedback. The central challenge is the saturation of the logistics link: large reward gaps produce nearly deterministic preferences and small sigmoid derivatives, causing standard analyses to pay worst-case curvature factors. We proposed \algo, a saturation insensitive algorithm that maintains a Bernoulli log-loss confidence set, selects a greedy target action and chooses an optimistic competitor through an objective tailored to dueling feedback.

\noindent\textbf{Limitations.}
\algo\ can be computationally demanding, as each round requires
empirical risk minimization and joint optimization over the action
space and the confidence set. Practical approximations, including
warm-started gradient updates and alternating optimization, are
discussed in Appendix~\ref{app:rb-implementation}, although our
current guarantees assume exact optimization. Developing more
efficient algorithms and tractable exact implementations that
retain saturation-insensitive guarantees remains an important
direction for future work.

\appendix
\section{Additional Notations}

The log-loss can be written in the canonical logistic form
\begin{align}
\label{eq:psi-ov}
\ell(o,\sigma(v))=\psi(v)-ov,
    \qquad
    \psi(v)=\log(1+e^v).
\end{align}
This representation is useful because the curvature of the loss, $\psi''(v)=\sigma'(v)$, explicitly captures the sensitivity of binary comparisons to the reward gap: comparisons are informative when the gap is moderate, but become statistically weak when the sigmoid link is saturated.

Using \eqref{eq:psi-ov}, we can represent it as a Bregman divergence, i.e., 
\begin{align}
\notag
    \bar \phi_r(z) &= \psi\big(\Delta_r(z)\big) - \psi\big(\Delta_{r^*}(z)\big)- \frac{e^{\Delta_{r^*}(z)}}{1+e^{\Delta_{r^*}(z)}} \big(\Delta_{r}(z)-\Delta_{r^*}(z)\big)\\\label{eq:Bregman}
    &= B_{\psi}\big(\Delta_{r}(z),\Delta_{r^*}(z)\big)
\end{align}

\section{Proof of Main Results}\label{app:proof_main}

\begin{lemma}\label{lemma:main_bernstein}
The event $\cE$ defined in~\eqref{eq:def_high_prob_event} holds with probability at least $1-\delta$.
\end{lemma}

We firstly provide a detailed description of Lemma \ref{lemma:main_bernstein}. As in \eqref{eq:def_high_prob_event}, we define the high-probability event $\cE$, which denotes the event where for any $r\in\cR$ and $t\in[T]$, we have
\begin{align}\label{eq:def_high_prob_event}
\sum_{i=1}^{t-1}\bar\phi_r(z_i)\le2\sum_{i=1}^{t-1}\phi_r(o_i, z_i)+2\beta_t.
\end{align}
On the event of $\cE$, since the left hand side is non-negative because each $\phi_r(z_i)$ is a KL-divergence, we have for any $r\in\cR$ and $t\in[T]$,
\begin{align*}
L_t(r)+\beta_t-L_t(r^*)=\sum_{i=1}^{t-1}\phi_r(o_i, z_i)+\beta_t\ge0.
\end{align*}
Therefore, $r^*\in\Conf_t$ for any $t\in[T]$, so we can use the property $J(a_t^*, r^*)\le J(b_t, r_2^t)$ to bound the term $r^*(x_t, a_t^*)$ in the regret. 

From now on, we suppose the event $\cE$ holds. We consider the following regret decomposition:
\begin{lemma}\label{lemma:per_step_regret}
    Suppose the event $\cE$ holds. For any $t$, let $r^t_1, r_2^t,z_t$ be defined in Algorithm \ref{algo:sicdb}. Then, the regret $R_t$ at step $t$ can be decomposed as
    \begin{align*}
        R_t \le \big|\Delta_{r_1^t}(z_t)-\Delta_{r^*}(z_t)\big|+2\big|\Delta_{r^*}(z_t)-\Delta_{r_2^t}(z_t)\big|.
    \end{align*}
\end{lemma}
To continue, we define $E_t = \max \big\{\big|\Delta_{r_1^t}(z_t)-\Delta_{r^*}(z_t)\big|, \big|\Delta_{r^*}(z_t)-\Delta_{r_2^t}(z_t)\big|\big\}$. For a constant $u>0$ to be determined, we consider the following ``good'' index sets:
\begin{align*}
    I_t := \big\{s \le t: E_t \le u\big\}. 
\end{align*}
We consider the regret on $I_T$ and $I_T^c$ separately, i.e.,
\begin{align*}
    \Regret(T) = \underbrace{\sum_{t \in I_T} R_t}_{(I_1)} + \underbrace{\sum_{t \in I_T^c} R_t}_{(I_2)}.
\end{align*}
\textbf{Bounding $(I_1)$}. Using Lemma \ref{lemma:per_step_regret}, we have
\begin{align}
\label{eq:qw0006}
    (I_1) \le \underbrace{\sum_{t \in I_T}\big|\Delta_{r_1^t}(z_t)-\Delta_{r^*}(z_t)\big|}_{(I_{1,1})} + 2 \underbrace{\sum_{t \in I_T}\big|\Delta_{r^*}(z_t)-\Delta_{r_2^t}(z_t)\big|}_{(I_{1,2})}.
\end{align}

Due to the definition of $I_T$, for any $t\in I_T$, we have $R_t\le 3E_t\le3u$. Moreover, we have
\begin{align*}
    |\Delta_{r^*}(z_t)| &= |r^*(x_t,a_t)-r^*(x_t,b_t)|\\
    &= \big|\big(r^*(x_t,a_t)-r^*(x_t,a_t^*)\big) + \big(r^*(x_t,a_t^*)-r^*(x_t,b_t)\big)\big|\\
    &\le \big|r^*(x_t,a_t)-r^*(x_t,a_t^*)\big| + \big|r^*(x_t,a_t^*)-r^*(x_t,b_t)\big|\\
    &= 2r^*(x_t,a_t^*) - r^*(x_t,a_t) - r^*(x_t,b_t) = R_t.
\end{align*}
Thus, we have $|\Delta_{r^*}(z_t)|\le R_t\le 3u$. For any $j\in\{1, 2\}$, we have the following upper bound for $|\Delta_{r_j^t}(z_t)|$:
\begin{align*}
\big|\Delta_{r_j^t}(z_t)\big|
&\le \big|\Delta_{r^*}(z_t)\big|+\big|\Delta_{r_j^t}(z_t)-\Delta_{r^*}(z_t)\big|\\
&\le 3u + E_t \le 4u.
\end{align*}
where the first inequality holds by the triangle inequality, and the second inequality holds by the definition of $E_t$. Next, we need the following lemma:
\begin{lemma}
\label{lemma:bregman-two-sided}
Let $\bar \phi_r$ be defined in \eqref{eq:excess}. For any $r\in\cR$ and $z\in\cZ$, there exists $\xi$ on the segment
between $\Delta_{r^*}(z)$ and $\Delta_r(z)$ such that
\begin{align*}
 \bar\phi_r(z)
    =\tfrac12\,\sigma'(\xi)\bigl(\Delta_r(z)-\Delta_{r^*}(z)\bigr)^2 .
\end{align*}
Consequently, for any $M\ge\max\{|\Delta_r(z)|,|\Delta_{r^*}(z)|\}$, we have the following inequality:
\begin{align*}
    \frac{\sigma'(M)}{2}\big(\Delta_r(z)-\Delta_{r^*}(z)\big)^2 \le \bar\phi_r(z)
    \le\frac{1}{8}\big(\Delta_r(z)-\Delta_{r^*}(z)\big)^2.
\end{align*}
\end{lemma}
Using this lemma with $M=4u$, we know that for any $j \in \{1,2\}$,
\begin{align}
\label{eq:qw0004}
    \big(\Delta_{r_j^{t}}(z_t)-\Delta_{r^*}(z_t)\big)^2 \le \frac{2}{\sigma'(4u)} \bar \phi_{r_j^t}(z_t)
\end{align}
Therefore, we have
\begin{align}
\notag
    (I_{1,j})&=\sum_{t \in I_T}\big|\Delta_{r_j^t}(z_t)-\Delta_{r^*}(z_t)\big|\\\notag
    &\le \sqrt{\sum_{t \in I_T}\big(\Delta_{r_j^t}(z_t)-\Delta_{r^*}(z_t)\big)^2 }\cdot\sqrt{T}\\\label{eq:qw0005}
    &\le \sqrt{\frac{2T}{\sigma'(4u)}} \cdot \sqrt{\sum_{t\in I_T} \bar \phi_{r_j^t}(z_t)},
\end{align}
where the first inequality holds due to the Cauchy-Schwarz inequality, and $|I_T|\le T$. The second inequality holds due to \eqref{eq:qw0004}. The next lemma shows an upper bound of $\sum_{i=1}^{t-1} \bar \phi_{r_j^t}(z_i)$, which is important to further control $(I_{1,j})$.
\begin{lemma}\label{lemma:good_event_regret} Suppose the high-probability event $\cE$ defined in Lemma \ref{lemma:main_bernstein} holds. Then, for any $t \in [T+1]$, $j \in \{1,2\}$, we have
\begin{align*}
    \sum_{i=1}^{t-1} \bar \phi_{r_j^t}(z_i) \le 4 \beta_T.
\end{align*}
\end{lemma}
Note that $\bar \phi_r(z) \ge 0$ because it's defined as a Bregman divergence. Moreover, $I_T \subseteq [T]$. Let $I_T = \{t_1<t_2<\ldots<t_K\}$, where $|I_T|=K$. Thus, for any $1\le k \le K$, we know that
\begin{align*}
    \sum_{l=1}^{k-1} \bar \phi_{r_j^{t_k}}(z_{t_l}) \le \sum_{i=1}^{t_k-1} \bar \phi_{r_j^{t_k}}(z_i) \le 4 \beta_T.
\end{align*}
Moreover, for any $k$, we have
\begin{align*}
    \bar \phi_{r_j^{t_k}}(z_{t_k}) &\le \frac{1}{8} \big(\Delta_{r_j^{t_k}}(z_{t_k})-\Delta_{r^*}(z_{t_k})\big)^2
\end{align*}
due to Lemma \ref{lemma:bregman-two-sided}. First, using the definition of $E_t$, we have
\begin{align*}
    \frac{1}{8} \big(\Delta_{r_j^{t_k}}(z_{t_k})-\Delta_{r^*}(z_{t_k})\big)^2&\le \frac{1}{8} (E_{t_k})^2
    \le \frac{u^2}{8},
\end{align*}
where the last inequality uses $t_k \in I_T$. Second, we have the following lemma
\begin{lemma}\label{lemma:self_bounded}
Let $\phi_r$ and $\bar\phi_r$ be defined in \eqref{eq:excess}. Then, for any reward function $r \in \cR$, any $(o,z) \in \{0,1\} \times \cX \times \cA \times \cA$, the following properties hold:
\begin{align*}
    |\phi_r(o, z)|\le2B. 
\end{align*}
Additionally, we have
\begin{align*}
    \bar\phi_r(z)=\EE_{o}[\phi_r(o, z)] \le 2B,  |\phi_r(o, z)-\bar\phi_r(z)|\le4B.
\end{align*}
\end{lemma}
Hence, $\bar \phi_{r_j^{t_k}}(z_{t_k}) \le 2B$. In conclusion, we have
\begin{align*}
    \phi_{r_j^{t_k}}(z_{t_k}) \le \min\{2B,u^2/8\} := \tau_u.
\end{align*}
Now, applying Lemma \ref{lemma:regret_card}, for any $k\in[K]$, $j \in \{1,2\}$, and $\epsilon > 0$, the following inequality holds: $d^{\epsilon,\eta}_{\elud}(\cR)$
    \begin{align}
    \label{eq:qw0007}
    \sum_{l=1}^k \phi_{r_j^{t_l}}(z_{t_l})\le\big(2B+\beta_T\log(1+2B/\epsilon)\big)d^{\epsilon,\tau_u}_{\elud}(\cR)+2B+K\epsilon.
    \end{align}
Now, set $\epsilon=1/T$, and define
\begin{align*}
    \Gamma_u = \big(2B+\beta_T\log(1+2BT)\big)d^{\tau_u}_{\elud}(\cR)+2B+1.
\end{align*}
Combining \eqref{eq:qw0006}, \eqref{eq:qw0005} and \eqref{eq:qw0007}, we have shown that
\begin{align}
\label{eq:I1}
    (I_1) &\le (I_{1,1}) + 2 (I_{1,2})\le 3 \sqrt{\frac{2T \Gamma_u}{\sigma'(4u)}}.
\end{align}
\textbf{Bounding $(I_2)$}. Using $r^*(\cdot) \in [-B,B]$, we directly have $(I_2) \le 4B |I_T^c|$. It suffices to bound $|I_T^c|$. Note that for any $t \in I_T^c$, $E_t = \max \big\{\big|\Delta_{r_1^t}(z_t)-\Delta_{r^*}(z_t)\big|, \big|\Delta_{r^*}(z_t)-\Delta_{r_2^t}(z_t)\big|\big\} > u$. Trivially, $\max\{|\Delta_r(z)|,|\Delta_{r^*}(z)|\} \le 2B$. Using Lemma \ref{lemma:bregman-two-sided}, for any $j \in \{1,2\}$, we have
\begin{align*}
    \bar\phi_{r_j^t}(z_t) \ge \frac{\sigma'(2B)}{2}\big(\Delta_{r_j^t}(z_t)-\Delta_{r^*}(z_t)\big)^2. 
\end{align*}
Therefore, 
\begin{align*}
    \max_{j \in \{1,2\}}\bar\phi_{r_j^t}(z_t)\ge \frac{\sigma'(2B)}{2} (E_t)^2 > \frac{\sigma'(2B)}{2} u^2 := \theta_u.
\end{align*}
This indicates
\begin{align}
\label{eq:qw0008}
    I_u^c
    \subseteq
    \bigcup_{j\in\{1,2\}}
    \bigg\{t\in[T]:\bar\phi_{r_j^t}(z_t)>\frac{\sigma'(2B)}{2} u^2\bigg\}.
\end{align}
Using Lemma \ref{lemma:good_event_regret}, for any $t \in [T+1]$, $j \in \{1,2\}$, we have
\begin{align*}
    \sum_{i=1}^{t-1} \bar \phi_{r_j^t}(z_i) \le 4 \beta_T.
\end{align*}
Now we can apply Lemma \ref{lemma:regret_card}. For every $\epsilon \in (0,\theta_u]$, $j \in \{1,2\}$, and every $t$, we have
\begin{align*}
    \sum_{i=1}^{t-1} \ind \big\{\bar\phi_{r_j^i}(z_i) > \epsilon\big\} \le \bigg(\frac{4\beta_T}{\epsilon} + 1\bigg) d^{\epsilon,\theta_u}_{\elud}(\cR) + 1.
\end{align*}
Using \eqref{eq:qw0008}, we have
\begin{align*}
    |I_u^c| &\le \sum_j \big|\big\{t\in [T]: \bar \phi_{r_j^t}(z_t) > \theta_u\big\}\big|\\
    &= \sum_j \sum_{t \in [T]}\ind \big\{\bar\phi_{r_j^t}(z_t) > \theta_u\big\}\\
    &\le 2\bigg(\frac{8\beta_T}{u^2 \sigma'(2B)} + 1\bigg) d^{\theta_u,\theta_u}_{\elud}(\cR) + 2.
\end{align*}
As a result, we have
\begin{align}
\label{eq:I2}
    (I_2) \le 4B\bigg(\frac{8\beta_T}{u^2 \sigma'(2B)} + 1\bigg) d^{\theta_u,\theta_u}_{\elud}(\cR)+ 8B.
\end{align}
In conclusion, combining \eqref{eq:I1} and \eqref{eq:I2}, we have
\begin{align*}
    \Regret(T) \le 3 \sqrt{\frac{2T \Gamma_u}{\sigma'(4u)}}+4B\bigg(\frac{8\beta_T}{u^2 \sigma'(2B)} + 1\bigg) d^{\theta_u,\theta_u}_{\elud}(\cR)+ 8B.
\end{align*}
Taking $u$ which can achieve the minimum, we have completed the proof of Theorem \ref{thm:main}.

\section{Proof of Lemmas in Appendix~\ref{app:proof_main}}\label{app:proof_A}

\subsection{Proof of Lemma~\ref{lemma:main_bernstein}}

\begin{proof}[Proof of Lemma~\ref{lemma:main_bernstein}]
We define a filtration $\cF_t$ to be the $\sigma$-algebra generated by $\{Z_s\}_{s\le t-1} \cup z_t$. Recall that $\phi_r(o,z)
    =
    \ell(o,f_r(z))-\ell(o,f_{r^*}(z))$.
Note that 
\begin{align*}
\EE[\phi_r(o_t,z_t)|\cF_{t}]=\EE_{o_t \sim \Ber(f_{r^*}(z_t))}[\phi_r(o_t,z_t)] = \bar \phi_r(z_t).
\end{align*}
To apply Lemma \ref{lemma:bernstein_basic}, we need the following lemmas:
\begin{lemma}
\label{lemma:variance_bounded}
    With the same notations as Lemma \ref{lemma:self_bounded}, we have
    \begin{align*}
        \Var_o[\phi_r(o, z)]\le(2B+4)\bar\phi_r(z).
    \end{align*}
\end{lemma}
\begin{proof}[Proof of Lemma \ref{lemma:variance_bounded}]
    It follows from Proposition 15 of \citet{bakhtiari2026eluder}.
\end{proof}
With Lemmas \ref{lemma:self_bounded} and \ref{lemma:variance_bounded}, we can apply Lemma \ref{lemma:bernstein_basic} with $b=2B$ and $c=2B+4$. Therefore, with probability at least $1-\delta$, for any $t\in[T]$ and $r\in\cR$, we have
\begin{align*}
\sum_{i=1}^{t-1}\bar\phi_r(z_i)\le2\sum_{i=1}^{t-1}\phi_r(o_i, z_i)+2\beta_t.
\end{align*}
Here, we define 
\begin{align*}
    \beta_t = \frac{5T\epsilon}{2} + 60(B+1) \log (N_\epsilon h_T/\delta),
\end{align*}
where $N_\epsilon = \cN\big(\Phi(\cR),\epsilon,\|\cdot\|_\infty\big)$, $h_T = e + \log(1+T)$. Moreover, if we set $\epsilon = 1/T$, we complete the proof of Lemma \ref{lemma:main_bernstein}.
\end{proof}

\subsection{Proof of Lemma \ref{lemma:per_step_regret}}
\begin{proof}[Proof of Lemma \ref{lemma:per_step_regret}]
    Since $r^*\in\Conf_t$, we can use the optimality condition of $(b_t, r_2^t)$ to bound the term $r^*(x_t, b_t)$, i.e.,
\begin{align*}
2r^*(x_t, a_t^*)-r_1^t(x_t, a_t^*)-2r^*(x_t, a_t)\le 2r_2^t(x_t, b_t)-r_1^t(x_t, b_t)-2r_2^t(x_t, a_t).
\end{align*}
Rearranging terms, we have
\begin{align}
2r^*(x_t, a_t^*)\le2r_2^t(x_t, b_t)-r_1^t(x_t, b_t)-2r_2^t(x_t, a_t)+r_1^t(x_t, a_t^*)+2r^*(x_t, a_t).
\label{eq:comp_action}
\end{align}
Substituting \eqref{eq:comp_action} into the definition of the single-step regret,we have the per-step regret bound
\begin{align}
    R_t 
    &= 2r^*(x_t,a_t^*)-r^*(x_t,a_t) - r^*(x_t,b_t) \notag\\
    &\le\big[2r_2^t(x_t, b_t)-r_1^t(x_t, b_t)-2r_2^t(x_t, a_t)+r_1^t(x_t, a_t^*)+2r^*(x_t, a_t)\big]-r^*(x_t, a_t)-r^*(x_t, b_t) \notag\\
    &=\Delta_{r^*}(z_t)-2\Delta_{r_2^t}(z_t)+\big[r_1^t(x_t, a_t^*)-r_1^t(x_t, b_t)\big] \notag\\
    &\leq
    \Delta_{r^*}(z_t)-2\Delta_{r_2^t}(z_t)+\big[r_1^t(x_t, a_t)-r_1^t(x_t, b_t)\big] \notag\\
    &= \Delta_{r_1^t}(z_t)-2\Delta_{r_2^t}(z_t)+\Delta_{r^*}(z_t) \notag\\
    &\le\underbrace{|\Delta_{r_1^t}(z_t)-\Delta_{r^*}(z_t)|}_{e_{1,t}}+2\underbrace{|\Delta_{r^*}(z_t)-\Delta_{r_2^t}(z_t)|}_{e_{2,t}},
    \label{eq:app-perstep-bound}
\end{align}
where the first inequality holds due to \eqref{eq:comp_action}, the second inequality holds due to the optimality of $a_t$ in the design of the algorithm, which indicates $r_1^t(x_t, a_t^*)\le r_1^t(x_t, a_t)$, and the last inequality holds due to the triangle inequality.
\end{proof}

\subsection{Proof of Lemma~\ref{lemma:bregman-two-sided}}
\begin{proof}[Proof of Lemma~\ref{lemma:bregman-two-sided}]
First, recall that in \eqref{eq:Bregman}
\begin{align*}
    \bar\phi_r(z)
    =B_\psi(\Delta_r(z),\Delta_{r^*}(z)),
\end{align*}
where $B_\psi$ denotes the Bregman divergence generated by the
logistic potential $\psi(v):=\log(1+e^v)$.
Note that $\psi'(v)=\sigma(v)$ and $\psi''(v)=\sigma'(v)$. Applying Taylor's theorem to $\psi$, we have
\begin{align}
\notag  \psi\big(\Delta_r(z)\big)
    &=\psi\big(\Delta_{r^*}(z)\big)
    +\psi'\big(\Delta_{r^*}(z)\big)\big(\Delta_r(z)-\Delta_{r^*}(z)\big)\\\label{eq:qw0003} 
    &\qquad
    +\tfrac12\,\psi''(\xi)\big(\Delta_r(z)-\Delta_{r^*}(z)\big)^2,
\end{align}
where $\xi$ is a point on the segment between $\Delta_{r^*}(z)$ and $\Delta_r(z)$. Using the definition of Bregman divergence, we have
\begin{align}
    &B_\psi\big(\Delta_r(z),\Delta_{r^*}(z)\big)
    =\psi\big(\Delta_r(z)\big)
    -\psi\big(\Delta_{r^*}(z)\big)
    -\psi'\big(\Delta_{r^*}(z)\big)\big(\Delta_r(z)-\Delta_{r^*}(z)\big)
    \notag\\
    &\qquad =\Big[\psi\big(\Delta_{r^*}(z)\big)
    +\psi'\big(\Delta_{r^*}(z)\big)\big(\Delta_r(z)-\Delta_{r^*}(z)\big)
    +\tfrac12\,\psi''(\xi)\bigl(\Delta_r(z)-\Delta_{r^*}(z)\bigr)^2\Big]
    \notag\\\notag
    & \qquad\qquad
    -\psi\big(\Delta_{r^*}(z)\big)
    -\psi'\big(\Delta_{r^*}(z)\big)\bigl(\Delta_r(z)-\Delta_{r^*}(z)\bigr)\\
    &\qquad = \tfrac12\,\psi''(\xi)\bigl(\Delta_r(z)-\Delta_{r^*}(z)\bigr)^2,
    \label{eq:bregman-equality-derivation-substituted}
\end{align}
where we use \eqref{eq:Bregman}.

For the last inequality, note that $\psi'' = \sigma'$. Thus, the upper bound suffices $\psi''(\xi) \le 1/4$. When 
$M\ge\max\{|\Delta_r(z)|,|\Delta_{r^*}(z)|\}$, since $\xi$ is on the segment between $\Delta_{r^*}(z)$ and $\Delta_r(z)$, i.e., there exists $\lambda \ge 0$ such that $\xi = \lambda \Delta_{r^*}(z) + (1-\lambda)\Delta_r(z)$, we have $|\xi| \le M$. We complete the proof by noting $\psi''(\xi) = \sigma'(\xi) \ge \sigma'(M)$.
\end{proof}

\subsection{Proof of Lemma \ref{lemma:good_event_regret}}
\begin{proof}[Proof of Lemma \ref{lemma:good_event_regret}]
Note that $r_2^t\in\Conf_t$ by the algorithm design. Moreover, $r_1^t\in\Conf_t$ holds trivially.
Therefore, on the event $\cE$, for any $t\in[T]$ and $j\in\{1, 2\}$, we have
\begin{align}
\sum_{i=1}^{t-1}\bar\phi_{r_j^t}(z_i)&\le2\sum_{i=1}^{t-1}\phi_{r_j^t}(o_t, z_t)+2\beta_t\nonumber\\
&=2\big(L_t(r_j^t)-L_t(r^*)+\beta_t\big)\nonumber\\
&\le2\Big(L_t(r_j^t)-L_t(r_1^t)+\beta_t\Big)\nonumber\\
&\le4\beta_t\le4\beta_T,\label{eq:total_excess_loss}
\end{align}
where the first inequality holds due to the definition of the high-probability event, the second inequality holds due to $r_1^t \in \argmin_{r \in \cR} L_t(r)$, the third inequality holds because $r_j^t\in\Conf_t$, and the last inequality holds due to the monotonicity of $\beta_t$.
\end{proof}

\subsection{Proof of Lemma \ref{lemma:self_bounded}}
\begin{proof}[Proof of Lemma \ref{lemma:self_bounded}]
Recall that
\begin{align*}
    \phi_r(o,z)
    &=
    \ell\big(o,f_r(z)\big)-\ell\big(o,f_{r^*}(z)\big), \\ \ell\big(o,f_r(z)\big) &= - o \log \sigma \big(\Delta_r(z)\big) - (1-o) \log \sigma \big(-\Delta_r(z)\big).
\end{align*}
Therefore, we have
\begin{align*}
\phi_r(o, z)=o\cdot\log\frac{1+\exp(-\Delta_{r}(z))}{1+\exp(-\Delta_{r^*}(z))}+(1-o)\cdot\log\frac{1+\exp(\Delta_{r}(z))}{1+\exp(\Delta_{r^*}(z))}.
\end{align*}
We notice that for any $r$, $\Delta_r(z)\in[-2B, 2B]$, so
\begin{align*}
-2B=\log\frac{1+\exp(-2B)}{1+\exp(2B)}\le\log\frac{1+\exp(-\Delta_{r}(z))}{1+\exp(-\Delta_{r^*}(z))}\le\log\frac{1+\exp(2B)}{1+\exp(-2B)}=2B.
\end{align*}
Similarly, we have
\begin{align*}
\log\frac{1+\exp(\Delta_{r}(z))}{1+\exp(\Delta_{r^*}(z))}\in[-2B, 2B].
\end{align*}
Therefore, $\phi_r(o, z)$ can be bounded by
\begin{align*}
|\phi_r(o, z)|\le\max\bigg\{\bigg|\log\frac{1+\exp(-\Delta_{r}(z))}{1+\exp(-\Delta_{r^*}(z))}\bigg|, \bigg|\log\frac{1+\exp(\Delta_{r}(z))}{1+\exp(\Delta_{r^*}(z))}\bigg|\bigg\} \le 2B.
\end{align*}
\end{proof}
%=====================================
\section{Dueling Bandits with Linear Reward}
\label{sec:app-glm}
In this section, we consider the specialized linear reward function class, i.e., 
\begin{align*}
\cR=\Big\{r_{\btheta}=\la \bmu(x,a), \btheta \ra \Big| \btheta \in \Theta \subseteq \RR^d\Big\},
\end{align*}
where $\bmu:\cX \times \cA \to \RR^d$ is a fixed feature map. We assume $\|\bmu(x,a)\|_2 \le 1$ for any $(x,a) \in \cX \times \cA$, and $\|\btheta\|_2 \le B$ for any $\btheta \in \Theta$. Consider the functor $\Phi$ and $\bar \Phi$ defined in~\eqref{eq:def_func_class_w_loss}. We will specify the uniform covering number of $\Phi(\cR)$ and the localized Eluder dimension of $\bar \Phi(\cR)$.
\subsection{Covering Number}
\begin{proposition}
\label{prop:glm-covering}
Under Definition xx with the uniform norm, the $\epsilon$-covering
number of $\Phi(\cR)$ can be upper bounded by
\begin{align*}
    \cN\big(\Phi(\cR),\epsilon,\|\cdot\|_\infty\big)
    \le\bigg(1+\frac{8B}{\epsilon}\bigg)^{\!d} .
\end{align*}
In particular, when $\epsilon=1/T$,
\begin{align*}
    N_T
    =
    \cN\big(\Phi(\cR),1/T,\|\cdot\|_\infty\big)
    \le
    \big(1+8B T\big)^d .
\end{align*}
\end{proposition}

\begin{proof}[Proof of Proposition~\ref{prop:glm-covering}]
The proof directly follows from the Lipschitzness of $\Phi$. To be more specific, for any $\btheta,\btheta' \in \Theta$, we have
\begin{align*}
    \|\Phi(r_{\btheta}) -\Phi(r_{\btheta'}) \|_{\infty} &= \sup_{o,z} \Big|\ell\big(o,\sigma[\Delta_{r_{\btheta}}(z)]\big)-\ell\big(o,\sigma[\Delta_{r_{\btheta'}}(z)]\big)\Big|\\
    &= \sup_{o,z} \Big| \psi \big(\Delta_{r_{\btheta}}(z)\big) - \psi \big(\Delta_{r_{\btheta'}}(z)\big) - o\big(\Delta_{r_{\btheta}}(z)-\Delta_{r_{\btheta'}}(z)\big)\Big|\\
    &\le \sup_{o,z} \big| \psi \big(\Delta_{r_{\btheta}}(z)\big) - \psi \big(\Delta_{r_{\btheta'}}(z)\big) \big| + \big|\Delta_{r_{\btheta}}(z)-\Delta_{r_{\btheta'}}(z)\big|\\
    &\le \sup_{z} 2\big|\Delta_{r_{\btheta}}(z)-\Delta_{r_{\btheta'}}(z)\big|,
\end{align*}
where we use $\ell\big(o,\sigma(v)\big)=\psi(v)-ov$. The first inequality holds due to the triangle inequality. The last inequality holds due to $\psi(v) = \log(1+e^v)$ is $1$-Lipschitz. Moreover, we have
\begin{align*}
    \sup_{z} 2\big|\Delta_{r_{\btheta}}(z)-\Delta_{r_{\btheta'}}(z)\big| &= \sup_{x,a,b} 2 \big|\la \bmu(x,a)-\bmu(x,b), \btheta - \btheta' \ra\big|\\
    &\le 4 \|\btheta - \btheta' \|_2.
\end{align*}
Therefore, $\|\Phi(r_{\btheta}) -\Phi(r_{\btheta'}) \|_{\infty} \le 4 \|\btheta - \btheta' \|_2$. Using the standard arguments, i.e, 
let $\{\btheta_i\}_{i=1}^N$ be an $\epsilon/4$-cover of $\Theta$ under the Euclidean norm. Then, for every $\btheta\in\Theta$, there exists some $i\in[N]$ such that
\begin{align*}
    \|\btheta-\btheta_i\|_2\le \epsilon/4,
\end{align*}
and hence
\begin{align*}
    \|\Phi(r_{\btheta})-\Phi(r_{\btheta_i})\|_\infty
    \le 4\|\btheta-\btheta_i\|_2
    \le \epsilon .
\end{align*}
Thus $\{\Phi(r_{\btheta_i})\}_{i=1}^N$ forms an $\epsilon$-cover of $\Phi(\cR)$ under $\|\cdot\|_\infty$. Therefore,
\[
    \cN\big(\Phi(\cR),\epsilon,\|\cdot\|_\infty\big)
    \le
    \cN\big(\Theta,\epsilon/4,\|\cdot\|_2\big).
\]
Since $\Theta\subseteq B_2(0,B)\subset\RR^d$, the standard Euclidean covering bound (see e.g. \citet{vershynin2018high}) gives
\begin{align*}
    \cN\big(\Theta,\epsilon/4,\|\cdot\|_2\big)
    \le
    \left(1+\frac{8B}{\epsilon}\right)^d .
\end{align*}
Consequently,
\begin{align*}
    \cN\big(\Phi(\cR),\epsilon,\|\cdot\|_\infty\big)
    \le
    \left(1+\frac{8B}{\epsilon}\right)^d .
\end{align*}
\end{proof}
\subsection{Localized Eluder Dimension}
\label{ssec:app-glm-setup}
For $z = (x,a,b)$, we use shorthand notation $\tilde \bmu(z) = \bmu(x,a)-\bmu(x,b)$. In this section, we consider 
\begin{align*}
    \bar\Phi(\cR)
    :=\Big\{
    z\mapsto
    \KL\!\big(
    \Ber \big[\sigma(\la \tilde \bmu(z),\btheta^*\rangle)\big]
    \big\|\,
    \Ber\big[\sigma(\la \tilde \bmu(z),\btheta\ra)\big]
    \big)
    :\btheta\in\Theta\Big\},
\end{align*}
where $\btheta^* \in \Theta$. Without loss of generality, we can assume $\Theta$ is convex. We first prove an upper bound of the localized Eluder dimension of $\bar\Phi(\cR)$, which follows the proof of Proposition 4 in \citet{bakhtiari2026eluder}. The main difference lies in that we consider the abstract action set $\cX \times \cA \times \cA$ via a fixed feature map $\bmu$, while they consider $\cA \in \RR^d$.

We have the following result: 
\begin{proposition}
\label{prop:glm-eluder-upper}
Assume $\Theta\subseteq \RR^d$ is convex. $\btheta^* \in \Theta$. For any $\btheta \in \Theta, \|\btheta\|_2 \le B$. There exists a universal constant $C>0$
such that for every
\begin{align}
    \notag0<\epsilon\le\eta\le\frac{\sigma'(2B)}{4} ,
\end{align}
we have
\begin{align*}
    \dim_{\rm elud}^{\eta}\big(\epsilon;\bar\Phi(\cR)\big)
    \le Cd\log\!\left(1+\frac{B^2}{\epsilon}\right) .
\end{align*}
\end{proposition}
To prove this result, we first need the following definition, first proposed in \citet{bakhtiari2026eluder}.
\begin{definition}[Witness Sequence]
For a function class $\cF$, $\omega > 0$, and a sequence of actions $(z_1,\ldots,z_k) \in (\cX \times \cA \times \cA)^k$. A sequence $(g_1,\ldots, g_k) \in \cF^k$ is called an $\omega$-witness sequence for $(a_1,\ldots, a_k)$ if for every $t \in \{1,\ldots, k\}$,
\begin{align*}
    \sum_{i=1}^{t-1} g_t(a_i) \le \omega \text{ and }g_t(a_t) \ge \omega.
\end{align*} 
\end{definition}
\begin{lemma}[{Adapted from \citealt[Lem.~36]{bakhtiari2026eluder}}]
\label{lemma:witness-rescaling}
Let
$(z_1,\ldots,z_k) \subseteq \cX \times \cA \times \cA$ be an $\omega$-eluder sequence with respect to
$\bar\Phi(\cR)$. 
Then there exists a witness sequence
$(\btheta_1,\ldots,\btheta_k)\subseteq\Theta$ such that
\begin{align}
    |\langle \tilde \bmu(z_i),\btheta_t-\btheta^*\rangle|\le 2\sqrt{\omega/\sigma'(2B)},
    \qquad i\le t\le k .
    \label{eq:witness-rescaling-bound}
\end{align}
In particular, if $\omega\le 1/(4/\sigma'(2B))$, then
\begin{align*}
|\langle \tilde \bmu(z_i),\btheta_t-\btheta^*\rangle|\le 1 \text{ for all }
i\le t\le k.
\end{align*}
\end{lemma}
For any $\epsilon$-eluder sequence $\{a_1,\ldots, a_k\}$, we can find a witness for it by definition. Using Lemma \ref{lemma:witness-rescaling}, we can find $\{\btheta_1,\ldots, \btheta_k\}$ another witness sequence such that
\begin{align*}
    |\langle \tilde \bmu(z_i),\btheta_t-\btheta^*\rangle|\le 1 \text{ for all }
i\le t\le k.
\end{align*}
The remaining proof follows the proof of Proposition 4 in \citet{bakhtiari2026eluder}, by considering the feature map $\tilde \bmu(z_j)$.
\begin{proof}[Proof of Corollary \ref{cor:linear}]
    If suffices to combine Proposition \ref{prop:glm-covering} and Proposition \ref{prop:glm-eluder-upper}.
\end{proof}
\subsection{Proof of Lemma~\ref{lemma:witness-rescaling}}
\begin{proof}[Proof of Lemma~\ref{lemma:witness-rescaling}]
Due to the definition, we can find $\{\btheta_1',\ldots, \btheta_k'\}$ as a witness sequence of $(z_1,\ldots, z_k)$, i.e., 
\begin{align*}
    &\bar \phi_{r_{\btheta_t'}}(z_t) \ge \omega,\\
    & \sum_{i=1}^{t-1} \bar \phi_{r_{\btheta_t'}}(z_i) \le \omega
\end{align*}
Note that $\bar \phi_{r}(z)$ is convex in $r$ and is minimized at $r_{\btheta^*}$, with $\bar \phi_{r_{\btheta^*}}(z) = 0$. Therefore, we can find $\lambda_t \in [0,1]$ such that
\begin{align*}
    \bar \phi_{r_{\btheta^*} + \lambda_t(r_{\btheta_t'}-r_{\btheta^*})}(z_t) = \omega.
\end{align*}
Easy to see, 
\begin{align*}
    r_{\btheta_*} + \lambda_t(r_{\btheta_t'}-r_{\btheta^*}) = r_{(\btheta^*+\lambda_t({\btheta_t' - \btheta^*}))}
\end{align*}
Since $\Theta$ is convex, we have $\btheta_t := \btheta^*+\lambda_t({\btheta_t' - \btheta^*}) \in \Theta$. By our construction, we have
\begin{align*}
    \bar \phi_{r_{\btheta_t}}(z_t) = \omega.
\end{align*}
Moreover, we have 
\begin{align*}
    \sum_{i=1}^{t-1} \bar \phi_{r_{\btheta_t}}(z_i) \le \sum_{i=1}^{t-1} \bar \phi_{r_{\btheta_t'}}(z_i) \le \omega.
\end{align*}
Therefore, $\btheta_t$ satisfies the same condition as $\btheta_t'$. And thus, $\{\btheta_t\}$ is also a witness sequence.

Using Lemma \ref{lemma:bregman-two-sided}, we have 
\begin{align*}
    \omega \ge \bar \phi_{r_{\btheta_t}}(z_j) &\ge \frac{\sigma'(2B)}{2} \big|\Delta_{\btheta_t}(z_j)-\Delta_{\btheta^*}(z_j)\big|^2\\
    &= \frac{\sigma'(2B)}{2} \big|\la \tilde \bmu(z_j), \btheta_t - \btheta^* \ra\big|.
\end{align*}
This completes the proof of the lemma.
\end{proof}

\section{Auxiliary Lemmas}
\begin{lemma}[Theorem 9 in~\citet{bakhtiari2026eluder}]\label{lemma:bernstein_basic}
Let $\cZ$ be a set, $\{Z_t\}_{t\in[T]}$ be a $\cZ$ valued process adapted to a filtration $\{\cF_t\}_{t\in[T]}$, and $\bPhi$ a set of real-valued functions on $\cZ$. Assume that
\begin{itemize}[leftmargin=*]
\item[1.] For some $b>0$, for all $\varphi\in\Phi$ and $t\in\NN_+$, $|\EE[\varphi(Z_t)|\cF_{t-1}]-\varphi(Z_t)|\le b$.
\item[2.] For some $c>0$, for all $\varphi\in\Phi$ and $t\in\NN_+$, $\Var[\varphi(Z_t)|\cF_{t-1}]\le c\EE[\varphi(Z_t)|\cF_{t-1}]$.
\end{itemize}
Let $\delta\in(0, 1)$, $\epsilon>0$, and $N$ be the $\epsilon$-covering number of $\bPhi$ in the uniform metric. For any $n\in\NN_+$, define
\begin{align*}
\beta(n, \delta, \epsilon, N)=\frac{5n\epsilon}2+15(b+c)\log(Nh_n/\delta),
\end{align*}
where $h_n=e+\log(1+n)$. Then with probability at least $1-\delta$, for all $\varphi\in\Phi$ and $n\in\NN_+$,
\begin{align*}
\sum_{t=1}^n\EE[\varphi(Z_t)|\cF_{t-1}]\le2\sum_{t=1}^n\phi(Z_t)+2\beta(n, \delta, \epsilon, N).
\end{align*}
\end{lemma}
\begin{lemma}[Lemma 24 in \citet{bakhtiari2026eluder}]\label{lemma:regret_card}
Fix $B, \beta>0$ and $0<u\le B$. Let $\cX$ be a set, and $\Psi$ a set of $[0, B]$-valued functions on $\cX$. Suppose sequences $(x_1, \dots, x_n)\in\cX^n$ and $(\psi_1, \dots, \psi_n)\in\Psi^n$ satisfy that for all $t\in[n]$,
\begin{align*}
\sum_{i=1}^{t-1}\psi_t(x_i)\le\beta.
\end{align*}
Then the following hold:
\begin{itemize}[leftmargin=*]
    \item[1.] For any $\epsilon\in(0, u]$ and any $t\in[n]$,
    \begin{align*}
    \sum_{i=1}^t\ind[\psi_i(x_i)>\epsilon]\le(1+\beta/\epsilon)\dim_{\mathrm{elud}}^u(\epsilon, \Psi)+1.
    \end{align*}
    \item[2.] If $\psi_t(x_t)\le u$ for all $t\in[n]$, then for any $\omega\in(0, u]$ and any $t\in[n]$,
    \begin{align*}
    \sum_{i=1}^t\psi_i(x_i)\le\big(B+\beta\log(1+B/\omega)\big)\dim_{\mathrm{elud}}^u(\omega, \Psi)+B+t\omega.
    \end{align*}
\end{itemize}
\end{lemma}

%<RB-A01>
\begingroup
\section{Connections to Prior Work}
\label{app:rb-supplement}

\subsection{Saturation Dependence in Existing Guarantees}
\label{app:rb-prior}

\paragraph{Global inverse-curvature dependence.}
For the BTL model, $\sigma'(v)=\sigma(v)(1-\sigma(v))$.
When rewards lie in $[-B,B]$, the worst-case inverse curvature is
\[
\kappa_B:=\sup_{|v|\le2B}\frac{1}{\sigma'(v)}
=\frac{1}{\sigma'(2B)}
=2+e^{2B}+e^{-2B}.
\]
Thus, replacing local derivatives by a uniform lower bound can
introduce exponential dependence on the reward scale.
Analyses of contextual preference and dueling bandits that use
global derivative bounds include
\citet{saha2021optimal,bengs2022stochastic,di2024variance}.
Reward-based guarantees for general or policy-induced reward classes
also exhibit saturation dependence, as discussed
in~\citet{li2024feel,xie2025exploratory,xiong2024iterative}.
The precise dependence differs across these works, but the use of
worst-case curvature is a common source of unfavorable reward-scale
dependence.

\paragraph{Local curvature in logistic and linear dueling bandits.}
In logistic bandits, a line of work improves this dependence by
retaining local curvature~\citep{faury2020improved,abeille2021instance,
faury2022jointly}. For feature vectors $\mathbf v_s$ and a linear
parameter $\btheta\in\mathbb R^d$, the regularized Hessian is
\[
H_t(\btheta)
=\lambda I+\sum_{s<t}\sigma'(\mathbf v_s^\top\btheta)
\mathbf v_s\mathbf v_s^\top,
\]
where $\lambda>0$. Its weights preserve the curvature at individual
observations, allowing confidence sets and exploration bounds to
adapt to local derivatives. For example, the Logistic-UCB-2
algorithm of~\citet{faury2020improved} achieves
\[
R_T^{\mathrm{logistic}}
=\widetilde O\!\left(d\sqrt T+\kappa_{\mathrm{log}}d^2\right),
\]
where $\kappa_{\mathrm{log}}$ is the worst-case inverse derivative
in their logistic-bandit model; the display suppresses dependence
on the fixed parameter-radius bound as well as logarithmic factors.
Here regret is measured in expected Bernoulli rewards. The
worst-case curvature appears in an additive term, rather than
multiplying the leading $\sqrt T$ term. Related advances for
generalized linear models include
\citet{lee2024unified,liu2024almost}.

In linear contextual dueling bandits, \citet{di2025nearly} use
estimates of the local link derivative to construct a
derivative-weighted covariance matrix. For sigmoid feedback,
their RCDB-S algorithm has a leading term
\[
\widetilde O\!\left(dB^{3/2}\sqrt T\right)
\]
without the worst-case inverse-derivative factor; their full bound
also accounts for adversarial feedback.
These methods exploit a known linear parameterization and matrix
arguments such as elliptical-potential bounds. Such tools are not
directly available for general reward classes. Our analysis instead
uses the complexity of the induced excess-loss class and a regret
decomposition tailored to the two queried actions.

\paragraph{Preference-based exploration in RLHF.}
\citet{chen2026avoiding} take a different route through
preference-based exploration and an iteratively updated comparator
sampler. For a fixed sampler $\pi_{\mathrm{sam}}$, their POPO
subroutine controls the preference-based regret
\[
\operatorname{Reg}_{\mathrm{pref}}(\pi_{\mathrm{sam}},T)
=\sum_{t=1}^T\EE\!\left[
P^*(y^*\succ y'\mid x)-P^*(y\succ y'\mid x)
\right],
\]
where $x$ follows the prompt distribution and, conditionally on $x$,
$y^*\sim\pi^*$, $y\sim\pi_t$, and
$y'\sim\pi_{\mathrm{sam}}$. Here $P^*$ is the true preference
probability and $\pi^*$ is the reward-optimal policy.
Their SE-POPO algorithm updates the sampler between successive
calls to POPO and uses a preference-to-reward reduction to obtain
reward guarantees with polynomial dependence on the reward scale.
Thus, preference-based regret is an intermediate tool in their
analysis. Our result instead controls cumulative two-armed
latent-reward regret through the asymmetric selection rule and
query-gap localization.

\subsection{Reduction to the Linear Case}
\label{app:rb-linear}
% Source: Xb6W A6--A7, OpenReview PDF page 9.
% Use rho_t for the ellipsoid radius, to distinguish it from the
% empirical-loss radius beta_t. This is a notation change only.

Consider the linear class in Example~\ref{example:linear}.
Write $r_1^t=r_{\hat\theta_t}$ and keep its greedy target $a_t$ fixed.
In parameter space, the loss-based confidence set is
\[
\{\theta\in\Theta:L_t(r_\theta)
\le L_t(r_{\hat\theta_t})+\beta_t\}.
\]
To expose the connection with an explicit bonus, suppose that this set
is replaced by the full ellipsoid
\[
\Conf_t
=\{\theta\in\mathbb R^d:
\|\theta-\hat\theta_t\|_{M_t}\le\rho_t\},
\qquad M_t\succ0.
\]
Here $\|v\|_{M_t}=(v^\top M_tv)^{1/2}$, and $\rho_t$ denotes the
ellipsoid radius, not the empirical-loss threshold $\beta_t$.
For $\db_t(b)=\bmu(x_t,b)-\bmu(x_t,a_t)$, the comparator objective is
\begin{align*}
J_t(b,r_\theta)
&=2\langle \db_t(b),\theta-\hat\theta_t\rangle
  +\langle\bmu(x_t,b),\hat\theta_t\rangle
  -2\langle\bmu(x_t,a_t),\hat\theta_t\rangle\\
&\le 2\rho_t\|\db_t(b)\|_{M_t^{-1}}
  +r_1^t(x_t,b)-2r_1^t(x_t,a_t).
\end{align*}
The last step is the Cauchy--Schwarz inequality in the $M_t$ norm.
Maximizing over the full ellipsoid attains this upper bound.
Dropping the term independent of $b$ therefore yields the comparator
score
\[
r_1^t(x_t,b)
+2\rho_t\|\bmu(x_t,b)-\bmu(x_t,a_t)\|_{M_t^{-1}}.
\]
With an unweighted Gram matrix, this has the
estimated-reward-plus-pairwise-width form of the CoLSTIM
comparator~\citep{bengs2022stochastic}. Using a derivative-weighted
matrix gives the corresponding local-curvature bonus form used by
RCDB-S~\citep{di2025nearly}. The original loss-based confidence set
retains curvature through the log loss instead of replacing it with
an explicitly chosen ellipsoid.

The comparison concerns the comparator score under the stated
ellipsoidal replacement, not equality of the confidence sets or
identity of the complete algorithms. Thus, the connection explains
the implicit curvature-aware exploration in our selection rule;
the localized analysis establishes its regret guarantee in the
general-class formulation.

\section{Practical implementation of \algo{}}
\label{app:rb-implementation}
% Source: 2c8k A1--A2 and the July 31 follow-up;
% OpenReview PDF pages 18--20. These are implementation suggestions,
% not a new complexity theorem or an empirical evaluation.

We outline a practical implementation of \algo{} that approximates
its optimization steps using gradient-based updates and alternating
optimization.

\paragraph{Reward-model update.}
For a parameterized nonlinear reward model, a gradient-based optimizer
such as Adam can be warm-started from the previous round's model.
A fixed number of updates offers a computationally inexpensive
approximation to solving the cumulative-loss problem from scratch.
A replay sample can also be used to approximate the historical loss,
with rescaling to the cumulative-loss scale.

\paragraph{Target-action update.}
For the finite action set in our model, the target can be found by
evaluating the current reward model on all candidate actions.
For a large or continuous differentiable action space, a practical
extension is to sample initial actions, retain the highest-scoring
ones, and run multiple gradient-ascent trajectories in action space.
The best resulting candidate is retained.

\paragraph{Comparator update.}
Fix the target $a_t$ and the fitted reward model $r_1^t$, and write
\[
g_t(r)=L_t(r)-L_t(r_1^t)-\beta_t.
\]
The comparator problem maximizes $J_t(b,r)$ subject to $g_t(r)\le0$.
Initialize the candidate model at $r_1^t$, which is feasible for the
loss-based constraint, and initialize the comparator action.
Then alternate between updating the model with the action fixed and
updating the action with the model fixed. Retain only candidates
that satisfy the confidence constraint.

The action update can use enumeration, or the multi-start procedure
above when applicable. For the model update, the penalized objective
proposed for minimization is
\[
-J_t(b,r)+\lambda g_t(r)+\frac{\rho}{2}[g_t(r)]_+^2,
\qquad [v]_+=\max\{v,0\},
\]
with multiplier $\lambda\ge0$ and penalty parameter $\rho>0$.
After a proposed update, check the confidence constraint.
If it is violated, backtracking along the parameter update can be
used to search for a feasible step. The multiplier update is
\[
\lambda\leftarrow\max\{0,\lambda+\eta_\lambda g_t(r)\},
\]
where $\eta_\lambda>0$ is a step size.

\endgroup

\bibliography{arxiv}
\bibliographystyle{ims}
\end{document}